\relax
\documentclass[letterpaper]{article} %DO NOT CHANGE THIS
\usepackage{ijcai19}  %Required
\usepackage{times}  %Required
\usepackage{helvet}  %Required
\usepackage{courier}  %Required
\usepackage{url}  %Required
\usepackage{graphicx}  %Required
\usepackage[title]{appendix}
\usepackage{graphicx}
\usepackage{algorithm}
\usepackage[noend]{algorithmic}
\usepackage{amsthm}
\usepackage{amssymb}
\usepackage{multicol}
\usepackage{amsmath}
\usepackage{stackrel}

\usepackage{times}
\usepackage{xcolor}
\usepackage{soul}

\usepackage[utf8]{inputenc}
\usepackage{caption}
\usepackage{subcaption}
\usepackage[moderate]{savetrees}

\title{Optimal Exploitation of Clustering and History Information in Multi-Armed Bandit}
\author{Djallel Bouneffouf, Srinivasan Parthasarathy,  Horst Samulowitz, Martin Wistuba \\
\affiliations{IBM T.J. Watson Research Center, Yorktown Heights, NY, USA } \\
\emails{\{djallel.bouneffouf@, spartha@us., samulowitz@us., martin.wistuba@\}ibm.com}
}
\newtheorem{thm}{Theorem}
\newtheorem{lem}[thm]{Lemma}
\newtheorem{fc}[thm]{Fact}

\newtheorem{claim}[thm]{Claim}

\DeclareMathOperator{\E}{\mathbb{E}}
\DeclareMathOperator{\defeq}{\overset{\text{def}}{=}}

\DeclareMathOperator*{\argmax}{arg\,max}

\renewcommand{\algorithmicrequire}{\textbf{Input:}}
\begin{document}

\maketitle

\begin{abstract}
We consider the stochastic multi-armed bandit problem and the contextual bandit problem with historical observations and pre-clustered arms. The historical observations can contain any number of instances for each arm, and the pre-clustering information is a fixed clustering of arms provided as part of the input. We develop a variety of algorithms which incorporate this offline information effectively during the online exploration phase and derive their regret bounds. In particular, we develop the META algorithm which effectively hedges between two other algorithms: one which uses both historical observations and clustering, and another which uses only the historical observations. The former outperforms the latter when the clustering quality is good, and vice-versa. Extensive experiments on synthetic and real world datasets on Warafin drug dosage and web server selection for latency minimization validate our theoretical insights and demonstrate that META is a robust strategy for optimally exploiting the pre-clustering information.
\end{abstract}

\section{Introduction}
Many sequential decision problems ranging from clinical trials to online ad placement can be modeled as multi-armed bandit problems (classical bandit)~\cite{newsrecommendation,Tang2013AutomaticAF}. At each time step, the algorithm chooses one of several possible actions and observes its reward with the goal of maximizing the cumulative reward over time.  A useful extension to classical bandit is the contextual multi-arm bandit problem, where before choosing an arm, the algorithm observes a context vector in each iteration~\cite{LangfordZ07}. In the conventional formulation of these problems, arms are assumed to be unrelated to each other and no prior knowledge about the arms exist. However, applications can often provide historical observations about arms and also similarity information between them prior to the start of the online exploration phase. In this work, we assume that similarity between the arms is given in the form of a fixed pre-clustering of arms %In particular, we study settings where the clustering information is supplied by domain experts or is based on data-driven clustering of arms based on historical information available prior to the start of the online exploration phase.
and we design algorithms which exploit the historical information and the clustering information opportunistically. In particular, 
we seek algorithms which satisfy the following property: if the quality of clustering is good, the algorithm should quickly learn to use this information aggressively during online exploration; however, if the quality of clustering is poor, the algorithm should ignore this information. We note at the outset
that it is possible to consider alternative formulations of the bandit problem in the presence of history and clustering information. For instance, one alternative is to use clustering metrics such as Dunn or Dunn-like indices~\cite{Liu:2010:UIC:1933307.1934597} to decide if the clustering information should be used during online exploration. However, in general, cluster validation metrics are tightly coupled with specific classes of clustering algorithms (e.g., distance vs. density based clustering); hence, we focus on bandit algorithms which are agnostic to specific clustering metrics and work with \textit{any given} clustering of the arms. Other alternatives include the design of algorithms for optimal pre-clustering of arms prior to online exploration, as well as incrementally modifying the clustering of the arms during online exploration. These are both valuable refinements to the problem we study but are beyond the scope of the current work. In this paper, 
we focus on the problem of \textit{optimal exploitation of the clustering information given as part of the input}.
% For example, in clinical trials it is often the case that one knows that patients belong to clusters with certain characteristics, in algorithm selection it is known that groups of algorithms behave similarly on given inputs, and in advertisement one often has relevant historical ad placement data available. 
%\noindent\textbf{Motivating Applications:} 
A real-world motivation for our work is the problem of dosing the drug Warfarin~\cite{sharabiani2015revisiting}. Correctly dosing Warfarin is a significant challenge since it is dependent on the patient’s clinical, demographic and genetic information. In the contextual bandit setting, this information can be modeled as the context vector and the bandit arms model the various dosage levels. The dosage levels are further grouped into distinct clusters (15 arms and 3 clusters in our dataset). Historical treatment responses for various dosage levels as well as clustering information derived from medical domain knowledge is available as part this problem instance. 
Another motivating application is web server selection for latency minimization in content distribution networks (CDNs)~\cite{CDN}. A CDN can choose mirrored content from several distributed web servers. The latencies of these servers are correlated and vary widely due to geography and network traffic conditions. This domain can be modeled as the classical bandit problem with the bandit arms representing various web servers, and latency being the negative reward. The web servers are further grouped into clusters based on their historical latencies (700 arms and 5 clusters in our dataset). These problem domains share the characteristics of the availability of historical observations to seed the online phase and grouping of arms into clusters either through domain knowledge or through the use of clustering algorithms.

\section{Related Work} 
Both the classical multi-armed bandit and the contextual bandit problems have been studied extensively along with their variants~\cite{LR85,UCB,AuerC98,dj10,Maillard2014,DB2019,Nguyen2014,Korda2016,Gentile2017,surveyDB,pandey2007multi,shivaswamy2012multi,RLbd2018,AllesiardoFB14,BRCF17}. The works which are closely related to ours are ~\cite{pandey2007multi} and ~\cite{shivaswamy2012multi}. \cite{pandey2007multi,wang2018regional} study the classical bandit problem under the same model of arm clustering as in this work, and ~\cite{shivaswamy2012multi} studies the classical bandit problem under the same model of historical observations as in this work. In contrast to ~\cite{pandey2007multi,wang2018regional,shivaswamy2012multi}, our work provides 1) algorithms which \textit{simultaneously} incorporate historical observations and clustering information in the classical bandit setting, 2) regret guarantees under tight and adversarial clustering for this setting, 3) algorithms which \textit{simultaneously} incorporate historical observations and clustering information in the contextual bandit setting; we also provide regret guarantees for our classical bandit algorithm which uses history; prior to this work, we are not aware of such extensions for the contextual bandit setting, and 4) the META algorithm which effectively hedges between the strategy that uses both clustering and historical observations vs. the strategy which uses only the historical observations and not the clustering information.

\section{Problem Setting}
%In this section, we define two new types of a bandit problems.

\noindent\textbf{Classical Bandit: }
The classical bandit problem is defined as exploring the response of $K$ arms within $T$ trials.
Playing an arm yields an immediate, independent stochastic reward according to some fixed unknown distribution associated with the arm whose support is in $(0,1)$.
The task is to find the reward-maximizing policy.

We adapt this scenario by assuming that the arms are grouped into $C$ clusters with arm $k$ assigned to cluster $c(k)$.
Unlike the classical setting, we also wish to incorporate historical observations which may be available for the arms. Specifically, for $t \in \{1, 2, \ldots, T\}$, let $r_k(t)$ denote the online reward from arm $k$ at time $t$.
For notational convencience, we set it to $0$ if $k$ was not played at time $t$.\footnote{This is purely for notational convenience, and does not affect the mean rewards computed per arm.}
Let $r^h_k(t)$ denote the historical reward for the $t^{th}$ instance of playing arm $k$ in history. $H_k$ is the number of historical instances available for arm $k$, and $H$ is defined as $H:=\sum^K_{k=1}H_k$. 
For each arm, the historical rewards are drawn independently from the same distribution as the online rewards.
Let $\theta_k$ denote the expected reward for arm $k$. The goal is to maximize the expected total reward during $T$ iterations $\E \left[\sum^T_{t=1} \theta_{k(t)} \right]$, where $k(t)$ is the arm played in step $t$, and the expectation is over the random choices of $k(t)$ made by the algorithm. An equivalent performance measure is the expected total regret, which is the amount of total reward lost by a specific algorithm compared to an oracle which plays the (unknown) optimal arm during each step. The expected total regret is defined as:
\begin{align}
\E[R(T)] \defeq \E \left[ \sum^T_{t=1}\left(\theta^*-\theta_{k(t)}\right) \right] = \sum_k \Delta_k \E[n_k(T)] \label{eqn:regret}
\end{align}
with $\theta^* \defeq \max_k \theta_k$,  $ \Delta_k \defeq \theta^* - \theta_k$, and $n_k(t)$ denotes the number of times arm $k$ has been played until $t$.

%One solution for the multi-armed bandit problem is the well known Upper Confidence Bound (UCB) algorithm \cite{UCB}. This algorithm plays the current best arm at every time step, where the best arm is defined as one which maximizes the sum of observed mean rewards and an uncertainty term. Algorithm \ref{alg:ucb} presents the pseudocode for UCB. The symbol $\hat{\theta}_{i, j}$ denotes the empirical mean reward for arm $i$ after it has been played $j$ times.
%\begin{algorithm}[!ht]
%\caption{UCB}
%\label{alg:ucb}
%\begin{algorithmic}
%  \STATE Play each arm once. For arm $i$, initialize %$\hat{\theta}_{i, 1}$ as the observed reward for arm $i$
%  \FOR{$t \in \{ K+1, K+2, \ldots, T\}$}
%      \STATE $k(t) \gets \argmax_i \hat{\theta}_{i, n_i(t)} + \sqrt{2 \log(t) / n_i(t)}$
%      \STATE Play arm $k(t)$ and observe its reward
%  \ENDFOR
%\end{algorithmic}
%\end{algorithm}

\noindent\textbf{Contextual Bandit: } 
The contextual bandit with history and pre-clustered arms is defined as follows:
At each step $t \in \{1,...,T\}$, the player is presented with a {\em context} ({\em feature vector}) $\vec{x}_t \in \mathbb{R}^d$ before playing an arm from the set $A = \{ 1, \ldots, N\}$ that are grouped into $C$ known clusters. Let $r_k(t)$ denote the reward which the player can obtain by playing arm $k$ at time $t$ given context $\vec{x}_t$. As in \cite{ChuLRS11}, we will primarily focus on bandits where
the expected reward of an arm is linearly dependent on the context. Specifically,
$\forall k, t: r_k(t) \in [0,1]$ and $\E[r_k(t) \vert \vec{x}_t] = \theta_k^\top \vec{x}_t$
where $\theta_k \in \mathbb{R}^d$ is an unknown coefficient vector associated with the arm $k$ which needs to be learned from data.
% As in the case of classical bandits (Section \ref{sec:ucb}), we assume that the dependencies among
% arms in a cluster can be described by a generative
% model, that has a known form but unknown parameters.  For example, in our experiments with synthetic data, we assume that the parameter $\theta_k$ of arm $k$ within cluster $i$ (here, $c(k) = i)$ is sampled i.i.d. from the multivariate Gaussian distribution $\mathcal{N}(\theta^c_i$, $ B^c_i)$, where $\theta^c_i \in \mathbb{R}^{d}$ is the mean of the distribution, and $B^c_i \in \mathbb{R}^{d \times d}$ is the covariance matrix. 
To incorporate historical information about rewards of the arms in a principled manner, we assume that 
each arm $k$ is associated with a matrix $\mathbf{H}_k \in \mathbb{R}^{d \times d}$ which records covariance information about the contexts played in history during times $t$ with $1\leq t \leq H_k$. $\mathbf{H}_k$ is computed using all observed context $\vec{x}_t$ so that:
\begin{align*}
& \mathbf{H}_{k}(t) = \mathbf{H}_{k}(t-1) + \vec{x}_{t} \vec{x}^{\top}_{t},
& \mathbf{H}_{k}(0) = \mathbf{I}_d, 
& \mathbf{H}_k = \mathbf{H}_k(H_k)
\end{align*}
Here, $\mathbf{I}_d$ is the identity matrix of dimension $d$. The arm is also associated with a vector parameter $b^h_{k}$ which is the weighted sum of historical contexts, where the weights are the respective rewards: 
\begin{align*}
& b^h_{k}(t) = b^h_{k}(t-1) + r^h_k(t) \vec{x}_{t},
& b^h_{k}(0)= \mathbf{0}
\end{align*}
The role of the covariance matrix in HLINUCBC is analogous to its role in LINUCB which in turn is analogous to its role in least squares linear regression.

\section{HUCBC for Classical Bandit}
\label{sec:HUCBC}
One solution for the classical bandit problem is the well known Upper Confidence Bound (UCB) algorithm \cite{UCB}. This algorithm plays the current best arm at every time step, where the best arm is defined as one which maximizes the sum of observed mean rewards and an uncertainty term. We adapt UCB for our new setting such that it can incorporate both clustering and historical information. Our algorithm incorporates the historical observations by utilizing it both in the computation of the observed mean rewards and the uncertainty term. Our algorithm incorporates the clustering information by playing at two levels: first picking a cluster using a UCB-like strategy at each time step, and subsequently picking an arm within the cluster, again using a UCB-like strategy.

The pseudocode for HUCBC is presented in Algorithm \ref{alg:HUCBC} and the mean reward for arm $k$ is defined in \eqref{eqn:muhatkh}.
%\noindent\begin{minipage}{0.5\linewidth}
\begin{equation}
\hat{\theta}_k^{h}(t) = \frac{\sum^{H_k}_{t'=1} r_{k}^{h}(t') + \sum_{t' = 1}^{t}r_k(t')}{H_k + n_k(t)} \label{eqn:muhatkh}
\end{equation}
%\end{minipage}%
%\begin{minipage}{0.5\linewidth}
\begin{equation}
\operatorname{HUCBC}_k(t) = \hat{\theta}_k^h(t) +\sqrt{\frac{2 \log(t+H_k)}{n_{k}(t)+H_k}} \label{eqn:HUCBCk}
\end{equation}
%\end{minipage}\par\vspace{\belowdisplayskip}
At each time step $t$ of the exploration, HUCBC computes the quantity $\operatorname{HUCBC}_k(t)$ for arm $k$ using (\ref{eqn:HUCBCk}).
$\operatorname{HUCBC}_k(t)$ represents an \textit{optimistic} estimate of the reward attainable by playing arm $k$: this estimate incorporates the
mean reward observed for the arm so far including the plays of the arm in history and an upper confidence term to account for the possibility of 
underestimation due to randomness in the rewards.
In a similar manner, HUCBC also computes $\hat{\theta}_i^{hc}(t)$ and $\operatorname{HUCBC}_i^c(t)$
for each cluster $i$:
%\noindent\begin{minipage}{0.5\linewidth}
\begin{equation}
\hat{\theta}_i^{hc}(t) = \frac{\sum^{H_i^c}_{t'=1} r_{i}^{hc}(t') + \sum_{t' = 1}^{t}r_i^c(t')}{H_i^c + n_i^c(t)} \label{eqn:muhatihc}
\end{equation}
%\end{minipage}%
%\begin{minipage}{0.5\linewidth}
\begin{equation}
  \operatorname{HUCBC}_i^c(t) = \hat{\theta}_i^{hc}(t) +\sqrt{\frac{2 \log(t+H_i^c)}{n_{i}^c(t)+H_i^c}} \label{eqn:HUCBCi}
\end{equation}
%\end{minipage}\par\vspace{\belowdisplayskip}
Note that $\hat{\theta}_i^{hc}(t)$ is the mean reward observed for cluster $i$ including plays of the cluster
in history and $\operatorname{HUCBC}_i^c(t)$ represents an \textit{optimistic} estimate of the reward attainable by playing cluster $i$.
The quantities $r_i^{hc}(t)$, $r_i^c(t)$, $n_i^c(t)$,  and $H_i^c$ in (\ref{eqn:muhatihc}) and (\ref{eqn:HUCBCi}) are the per-cluster analogues of the corresponding quantities defined per-arm. 
Also note that, while the per-cluster quantities carry a superscript $c$ as part of their notation, the per-arm quantities do not. At each time step $t$, HUCBC first picks the cluster which maximizes $\operatorname{HUCBC}_i^c(t)$ and then picks the arm within this cluster which maximizes $\operatorname{HUCBC}_k(t)$. %It is easy to see that this strategy methodically incorporates both the clustering and historical information.
\begin{algorithm}[H]
   \caption{The HUCBC algorithm}
\label{alg:HUCBC}
 \begin{algorithmic}
 	\STATE At time $t$, select cluster $i$ that maximizes $\operatorname{HUCBC}_i^c(t)$ in (\ref{eqn:HUCBCi}) and play arm $k$ in cluster $i$ that maximizes $\operatorname{HUCBC}_k(t)$ in (\ref{eqn:HUCBCk})
  \end{algorithmic}
\end{algorithm}
\noindent\textbf{Regret Analysis: }
We upper-bound the expected number of plays of a sub optimal cluster under tight clustering of arms. Let $i^*$ denote the cluster containing the best arm. Arms are said to be \textit{tightly clustered} if: 1) for each cluster $i$, there exists an interval $[l_i, u_i]$ which contains the expected reward of every arm in cluster $i$, and 2) for $i \neq i^*$, the intervals $[l_i, u_i]$ and $[l_{i^*}, u_{i^*}]$ are disjoint. For any cluster $i \neq i^*$, define $\delta_i = l_{i^*} - u_i$.
\begin{thm}[Tight Clustering]
\label{theo:tight}
The expected regret $\mathbb{E}[R(T)]$ at any time horizon $T$ under tight clustering of the arms in HUCBC is at most the following:
\begin{gather}
\sum_{i \neq i^*} \left(1 + \ell_i^c + \frac{\pi^2(1 + 6H_i^c)}{6(2H_i^c + 1)^2} + \frac{\pi^2(1 + 6H_{i^*}^c)}{6(2H_{i^*}^c + 1)^2} \right) \nonumber \\
+ \sum_{k \vert c(k) = i^*} \left(1 + \ell_k + \frac{\pi^2(1 + 6H_k)}{6(2H_k + 1)^2} + \frac{\pi^2(1 + 6H_{k^*})}{6(2H_{k^*} + 1)^2} \right)\label{eqn:HUCBCtight}
\end{gather}
Here, $\ell_i^c = \max\left(0, \frac{8\log(T + H_i^c)}{\delta_i^2} - H_i^c\right)$, 
$\ell_k = \max\left(0, \frac{8\log(T + H_k)}{\Delta_k^2} - H_k\right)$ and $k^* = \arg\max_k \theta_k$.
\end{thm}
\noindent\textbf{Intuition behind Theorem \ref{theo:tight}: }
First, the cumulative regret up to time $T$ is logarithmic in $T$.
Second, as the separation between clusters increases resulting in increased $\Delta_k$ values, the regret decreases. Third and a somewhat subtle aspect of this regret bound is that the number of terms in the summation equals the number of arms in the optimal cluster + the number of sub-optimal clusters. This
number is always upper bound by the total number of arms. In fact, the difference between them can be pronounced when the clustering is well balanced -- for example, when
there are $\sqrt{n}$ clusters with $\sqrt{n}$ arms each, the total number of terms in the summation is $2\sqrt{n}$ while the total number of arms is $n$.
When the quality of clustering is good, this is exactly the aspect of the regret bound which tilts the scales in favor of HUCBC compared to UCB or HUCB.

\begin{proof}
The proof consists of three steps whose ideas are as follows. In the first step, we decompose the HUCBC regret into two parts: regret contributed by sub-optimal clusters (i.e., clusters without the optimal arm) and regret contributed by the sub-optimal arms within the optimal cluster (i.e, the cluster containing the optimal arm). The latter quantity can be upper-bounded using known results for HUCB. The main challenge in our analysis 
is to upper-bound the former, which we accomplish in Step 2. The main idea in Step 2 is to develop a concentration inequality for
the reward obtained from any cluster. We accomplish this by proving that the cumulative reward deviation (i.e., the cumulative sum of actual rewards minus their expectation) is a martingale. In the analysis of UCB and HUCB, a concentration inequality of this type is derived through the 
use of Chernoff bounds: these are not applicable in our analysis due to the fact that rewards from a cluster are not only non-stationary but also 
not independent and highly correlated with rewards from previous plays of the cluster. We circumvent this difficulty through the use of martingale concentration
inequalities in place of Chernoff bounds. In Step 3, we tie the results of Steps 1 and 2 together to obtain our final regret bound for HUCBC.
The formal proof is as follows.
\noindent\textbf{Step 1:} Since HUCBC uses HUCB to play arms within each cluster, we start by introducing the following regret bound for HUCB~\cite{shivaswamy2012multi}.
\begin{fc}[Theorem 2,~\cite{shivaswamy2012multi}]
\label{fact:hucb}
The expected number of plays of any sub optimal arm $k$, within any cluster $i$, for any time horizon $T$, for any clustering of arms, satisfies: 
\begin{gather}
\E[n_k(T)] \leq 1 + \ell_k + \frac{\pi^2(1 + 6H_k)}{6(2H_k + 1)^2} + \frac{\pi^2(1 + 6H_{k^*})}{6(2H_{k^*} + 1)^2} \label{eqn:hucbregret}
\end{gather}
Here $k^*$ is the best arm within cluster $i$, and $\ell_k = \max\left(0, \frac{8\log(T + H_k)}{\nabla_k^2} - H_k\right)$, where $\nabla_k = \theta_{k^*} - \theta_k$.
\end{fc}
For any time horizon, the number of plays of a sub optimal arm in HUCBC is an upper bound on its regret. This quantity equals the sum of the number of plays of a sub optimal cluster along with the number of plays of sub optimal arms within the optimal cluster. The latter quantity is upper-bounded using Fact \ref{fact:hucb} since for any sub optimal arm $k$ in the optimal cluster $i^*$, $\Delta_k = \nabla_k$. The former quantity is upper bounded in \textbf{Steps 2} and \textbf{3}.

\noindent\textbf{Step 2:} Given cluster $i$, let $t_{\ell}$ be the time slot when cluster $i$ is played for the $\ell^{th}$ time in the online phase. 
Suppose $t_1, t_2, \ldots$ be fixed and given. Define:
\begin{gather*}
z_i^c(\ell) = 
  \begin{cases}
    0  & \text{if } \ell = 0 \\
    z_i^c(\ell - 1) + r_i^{hc}(\ell) - \theta_{j^h(\ell)} & \text{if } 1 \leq \ell \leq H_i^c \\
    z_i^c(\ell - 1) + r_i^c(\ell - H_i^c) - \theta_{j(\ell - H_i^c)} & \text{if } \ell > H_i^c
  \end{cases}
 \end{gather*}
In this definition, $j^h(\ell)$ is the arm that was played during the $\ell^{th}$ instance of cluster $i$ in the historical phase, and 
$j(\ell)$ is the arm that is played during the $\ell^{th}$ instance of cluster $i$ in the online phase, $r$ (with the appropriate subscripts and superscripts)
refers to the actual reward obtained for that instance, and $\theta$ (with the appropriate subscripts and superscripts) refers to the mean
reward of the arm for that instance.
We claim that the random sequence $\{z_i^c(\ell) \vert \ell = 1, 2, \ldots\}$ is a martingale. Indeed, for $1 \leq \ell \leq H_i^c$, we have:
\begin{gather*}
\E[z_i^c(\ell) \vert z_i^c(\ell-1) = z] = z + \E[r_i^{hc}(\ell) - \theta_{j^h(\ell)} \vert z_i^c(\ell-1) = z] \\
= z + \sum_{j \vert c(j) = i} \Pr[j^h(\ell) = j \vert z_i^c(\ell-1) = z] \\
\E[r_j^h(\ell) - \theta_j \vert z_i^c(\ell-1) = z \bigwedge j^h(\ell) = j]  = z
\end{gather*}
A similar argument holds when $\ell > H_i$ which implies $\{z_i^c(\ell) \vert \ell = 1, 2, \ldots\}$ is a martingale.
It is easy to verify that $\forall \ell,~\vert z_i^c(\ell) - z_i^c(\ell - 1) \vert < 1$; hence, by Azuma-Hoeffding inequality~\cite{alonspencer:prob}, we have:
\begin{gather}
% \forall \ell \geq 1, \forall v > 0: ~
\Pr[\vert z_i^c(\ell) \vert \geq v] = \Pr[\vert z_i^c(\ell) - z_i^c(0)\vert \geq v] \leq 2e^{-\frac{v^2}{2\ell}} \label{eqn:azumaHUCBC}
\end{gather}
Consider $i = i^*$. We note that since $t_1, t_2, \ldots$ are fixed, $\forall t$, $n_i^c(t)$ is a fixed number. We have: 
\begin{gather}
\text{(\ref{eqn:azumaHUCBC})} \implies \forall t: ~\Pr\left[HUCBC^c_{i^*}(t) \leq l_{i^*} \right] \nonumber \\
= \Pr\left[(n_{i^*}^c(t) + H_{i^*}^c) HUCBC^c_{i^*}(t) \leq (n_{i^*}^c(t) + H_{i^*}^c) l_{i^*} \right] \nonumber \\
\leq \Pr\Bigg[\sum_{\ell = 1}^{H_{i^*}^c} r_{i^*}^{hc}(\ell) + \sum_{\ell = 1}^{n_{i^*}^c(t)} r_{i^*}^c(\ell) + \sqrt{2(n_{i^*}^c(t) + H_{i^*}^c)\log(t + H_{i^*}^c)} \nonumber \\
\leq \sum_{k \vert c(k) = i^*} (n_k(t) + H_k)\theta_j \Bigg] \nonumber \\
= \Pr\left[z_{i^*}^c(n_{i^*}^c(t) + H_{i^*}^c) \leq - \sqrt{2(n_{i^*}^c(t) + H_{i^*}^c)\log(t + H_{i^*}^c)} \right] \nonumber \\
\leq \Pr\left[\vert z_{i^*}^c(n_{i^*}^c(t) + H_{i^*}^c) \vert \geq \sqrt{2(n_{i^*}^c(t) + H_{i^*}^c)\log(t + H_{i^*}^c)} \right] \nonumber \\
\stackrel{\text{(\ref{eqn:azumaHUCBC})}}{\leq} \frac{2}{(t + H_{i^*}^c)^2} \label{eqn:condii*}
\end{gather}
We note that since (\ref{eqn:condii*}) holds conditionally for any fixed sequence $t_1, t_2, \ldots$, it holds unconditionally as well.
Now consider $i \neq i^*$, and any fixed $t$ such that $n_i^c(t) = \gamma$, where $\gamma \geq \ell_i^c ~\defeq~ \max\left\lbrace0, \frac{8\log(t + H_i^c)}{\delta_i^2} - H_i^c\right\rbrace$.
\begin{gather}
\text{(\ref{eqn:azumaHUCBC})} \implies \Pr\left[HUCBC^c_{i}(t) \geq u_i + \delta_i \bigg\vert n_i^c(t) = \gamma \right]  \nonumber \\
= \Pr\left[(\gamma + H_i^c) HUCBC^c_{i}(t) \geq (\gamma + H_i^c)(u_i + \delta_i) \bigg\vert n_i^c(t) = \gamma \right] \leq \nonumber \\
\Pr\Bigg[\sum_{\ell = 1}^{H_i^c}r_i^{hc}(\ell) + \sum_{\ell = 1}^{\gamma}r_i^c(\ell) + \sqrt{2(\gamma + H_i^c)\log(t + H_i^c)} \nonumber \\
\geq \sum_{k \vert c(k) = i} (H_k + n_k(t))\theta_k + (\gamma + H_i^c)\delta_i \bigg\vert n_i^c(t) = \gamma \Bigg] \nonumber \\
= \Pr\Bigg[z_i^c(\gamma + H_i^c) \geq (\gamma + H_i^c)\delta_i - \sqrt{2(\gamma + H_i^c)\log(t + H_i^c)} \nonumber \\
\bigg\vert n_i^c(t) = \gamma \Bigg] 
\leq \Pr\Bigg[\vert z_i^c(\gamma + H_i^c) \vert \geq \sqrt{2(\gamma + H_i^c)\log(t + H_i^c)} \nonumber \\
\bigg\vert n_i^c(t) = \gamma \Bigg] 
\stackrel{\text{(\ref{eqn:azumaHUCBC})}}{\leq} \frac{2}{(t + H_i^c)^2} \label{eqn:condii}
\end{gather}
We note that since (\ref{eqn:condii}) holds conditionally for any fixed sequence $t_1, t_2, \ldots$, 
% and for any fixed $t$ with fixed $n_i^c(t) = \gamma$, where $\gamma \geq \ell_i^c$, it also holds unconditionally on $t_1, t_2, \ldots$, for any fixed $t$ conditioned on $n_i^c(t) \geq \ell_i^c$.
% In other words, 
we have:
\begin{gather}
% \forall t: ~
\Pr\left[HUCBC^c_{i}(t) \geq u_i + \delta_i \bigg\vert n_i^c(t) \geq \ell_i^c \right] \leq \frac{2}{(t + H_i^c)^2} \label{eqn:uncondii}
\end{gather}
\noindent\textbf{Step 3:} Suppose we have an event $\mathcal{A}$ such that $\mathcal{A} \implies \mathcal{B} \vee \mathcal{C} \vee \mathcal{D}$. Then, for any event $\mathcal{E}$, we have:
\begin{gather*}
\mathcal{A} \implies (\mathcal{A} \wedge \mathcal{E}) \vee (\mathcal{A} \wedge \overline{\mathcal{E}}) 
\implies (\mathcal{A} \wedge \mathcal{E}) \vee ((\mathcal{B} \vee \mathcal{C} \vee \mathcal{D}) \wedge \overline{\mathcal{E}}) \nonumber \\
\text{Hence, } \Pr[\mathcal{A}]  \leq \Pr[\mathcal{B} \wedge \overline{\mathcal{E}}] + \Pr[\mathcal{C} \wedge \overline{\mathcal{E}}] + \Pr[\mathcal{A} \wedge \mathcal{E}] + \Pr[\mathcal{D} \wedge \overline{\mathcal{E}}] \nonumber \\
\leq \Pr[\mathcal{B}] + \Pr[\mathcal{C}] + \Pr[\mathcal{A} \vert \mathcal{E}] + \Pr[\mathcal{D} \vert \overline{\mathcal{E}}]
\end{gather*}
Consider a suboptimal cluster $i$. If $H_i^c = 0$, then HUCB (like UCB) will initialize itself by playing cluster $i$ once at the start of the online phase. Define this as
event $\mathcal{B}$. Define $\mathcal{A} = \mathbf{1}(i(t) = i)$, $\mathcal{C} = \mathbf{1}(HUCBC^c_{i^*}(t) \leq l_{i^*})$, $\mathcal{D} = \mathbf{1}(HUCBC^c_{i}(t) \geq u_{i} + \delta_i)$, and 
$\mathcal{E} = \mathbf{1}\left(n_i(t) < \ell_i^c \right)$ above. We now have:
\begin{gather}
\E[n_i^c(T)] = \sum_{t = 1}^{T}\Pr[i(t) = i] \nonumber 
\leq 1 + \sum_{t = 1}^{T} \Pr\left[HUCBC^c_{i^*}(t) \leq l_{i^*}\right] \\+ \sum_{t = 1}^{T} \Pr\left[i(t) = i \bigg\vert n_i^c(t) < \ell_i^c \right] 
\nonumber \nonumber \\ 
+ \sum_{t = 1}^{T} \Pr\left[HUCBC^c_{i}(t) \geq u_{i} + \delta_i \bigg\vert n_i^c(t) \geq \ell_i^c \right] \nonumber\\
\\\leq 1 + \ell_i^c  + \sum_{t = 1}^{T} \Pr\left[HUCBC^c_{i^*}(t) \leq l_{i^*}\right] \nonumber\\
+ \sum_{t = 1}^{T} \Pr\left[HUCBC^c_{i}(t) \geq u_{i} + \delta_i \bigg\vert n_i^c(t) \geq \ell_i^c \right] \nonumber \\
\stackrel{\text{(\ref{eqn:condii*}) and (\ref{eqn:uncondii})}}{\leq} 1 + \ell_i^c + \sum_{t = 1}^{T} \frac{2}{(t + H_i^c)^2} + \sum_{t = 1}^{T} \frac{2}{(t + H_{i^*}^c)^2} \nonumber \\\leq 1 + \ell_i^c + \frac{\pi^2(1 + 6H_i^c)}{6(2H_i^c + 1)^2} + \frac{\pi^2(1 + 6H^c_{i^*})}{6(2H^c_{i^*} + 1)^2} \label{eqn:HUCBCboundtight}
\end{gather}
The theorem now follows by combining (\ref{eqn:HUCBCboundtight}) and (\ref{eqn:hucbregret}).
\end{proof}

We now upper-bound the expected number of plays of a sub optimal cluster where an adversary is free to cluster the arms in order to elicit the worst case behavior from HUCBC. For ease of analysis, we will analyze a variant of HUCBC which we denote as $HUCBC'$, which plays UCB at the inter-cluster level and plays HUCB at the intra-cluster level.% We defer the adversarial analysis of unmodified HUCBC to a subsequent version of this work.
\begin{thm}[Adversarial Clustering]
\label{theo:adversarial}
The expected regret at any time horizon $T$ under any clustering of the arms in $HUCBC'$ satisfies the following:
\begin{gather}
\mathbb{E}[R(T)] \leq \sum_{i \neq i^*} \max_{k \vert c(k) = i}\left(\frac{16r\log{T}}{(\Delta_k/2)^2}+ 2s +\frac{\pi}{3} \right) \nonumber \\
+ \sum_{k \vert c(k) = i^*} \left(1 + \ell_k + \frac{\pi^2(1 + 6H_k)}{6(2H_k + 1)^2} + \frac{\pi^2(1 + 6H_{k^*})}{6(2H_{k^*} + 1)^2} \right) \label{eqn:HUCBCprimeadversarial}
\end{gather}
Here, $r, s$ are constants and $\ell_k$ and $k^*$ are defined as in Theorem \ref{theo:tight}. 
\end{thm}
\begin{proof}
(\ref{eqn:hucbregret}) implies that if we have a cluster which contains $\alpha$ arms, then the expected number of plays of sub optimal arms within that cluster as incurred by the HUCB strategy is $O(\alpha \log T)$ for any time horizon $T$. This fact in turn implies that the expected per-step regret converges in to $0$, and hence the expected reward of this cluster converges to the reward of the best arm within this cluster:

\begin{align}
\lim_{t\to\infty} \E[\hat{\theta}_i^c(t)] = \max_{k \vert c(k) = i} \theta_k \label{eqn:drift1}
\end{align}

Further, set $v = \frac{4\log(t)}{\ell}$ in Eqn (\ref{eqn:azumaHUCBC}). Then, for any time horizon $t > 1$ it follows that total reward 
of a cluster after $\ell$ pulls is within the interval of width $2v$ centered around its expectation with probability at least $1 - \frac{1}{t^4}$.
This high concentration condition along with Eqn (\ref{eqn:drift1}) satisfy the drift conditions for non-stationary arms in 
\cite{kocsis2006bandit}. 

Suppose the UCB algorithm is applied to a non-stationary collection of `drifting' arms where the 
drift conditions are as above. Then, from \cite{kocsis2006bandit}, we know that the expected number of plays of any sub optimal `drifting' arm $i$, for any time horizon $T$ satisfies:
\begin{gather}
E[n_i^c(T)] \leq \frac{16r\log{T}}{((\theta_{i^*}^c - \theta_i^c)/2)^2}+ 2s +\frac{\pi}{3} \label{eqn:driftingucb}
\end{gather}
where $r$ and $s$ are constants, and $\theta_i^c$ is the steady state reward for the `drifting' arm $i$.
In $HUCBC^{\prime}$, the role of the drifting arms is played by the clusters.
The regret of $HUCBC^{\prime}$ can be decomposed into the regret from the sub-optimal arms in the optimal cluster + the regret due to the sub-optimal clusters.
The former is upper-bounded by (\ref{fact:hucb}) and the latter is upper-bounded by (\ref{eqn:driftingucb}). Combining them yields the theorem.
\end{proof}

\noindent\textbf{Comparing Theorems \ref{theo:tight} and \ref{theo:adversarial}: } Both these theorems share important similarities. First, 
the regret is $O(\log{T})$. Second, as the distance between clusters increase, the regret decreases. Third, the number of terms in the summation equals
the number of arms in the optimal cluster + the number of sub-optimal clusters. However, there are significant differences. First, the distance
between two clusters in Theorem \ref{theo:adversarial} is measured differently: it is now the difference
between the mean rewards of the best arms in the clusters. Second and more importantly, the constants involved in the bound of Theorem \ref{theo:adversarial} arise from the results of \cite{kocsis2006bandit} and are significantly bigger than those involved in the bound of Theorem \ref{theo:tight}. We emphasize that Theorem \ref{theo:adversarial} establishes only an upper bound on the regret: the actual regret for typical instances that arise in practice can be a lot smaller than this upper bound.
The proof of this theorem leverages the martingale concentration inequality derived as part of Theorem \ref{theo:tight}. This is a key part of the drift condition which needs to be satisfied in order for us to use the result of \cite{kocsis2006bandit} for bounding the performance of UCB in non-stationary environments.
\section{HLINUCBC for Contextual Bandit}
\label{sec:HLINUCBC}
A well known solution for the contextual bandit with linear payoffs is LINUCB~\cite{13} where the key idea is to apply online ridge regression to the training data to estimate the coefficients $\theta_k$.
We propose HLINUCBC (Algorithm \ref{alg:HLINUCBC}) which extends this idea with both historical and clustering information.

\noindent\textbf{ Clustering Information: } HLINUCBC deals with arms that are clustered; it applies online ridge regression at the per-cluster level to obtain an estimate of the coefficients $\hat{\theta}_i^c$ for each cluster (Line 7), plays the best cluster and then applies online ridge regression to the arms within the chosen cluster to obtain an estimate of the coefficients $\hat{\theta}_k$ for each arm (Line 6). To find the best cluster, at each trial $t$, HLINUCBC computes the quantities:
%\noindent\begin{minipage}{0.5\linewidth}
\begin{equation}
\begin{aligned}
p^{c}_{t,i} \leftarrow \hat{\theta}^{c\top}_i x_{t} + \alpha \sqrt{x^{\top}_{t} (\textbf{A}^c_{i})^{-1} x_{t}}\label{eqn:HLINUCBCcl}
\end{aligned}
\end{equation}
%\vspace{-20pt}
%\end{minipage}%
%\begin{minipage}{0.5\linewidth}
\begin{equation}
p_{t,k} \leftarrow \hat{\theta}^{\top}_k x_{t} + \alpha \sqrt{x^{\top}_{t}(\textbf{A}_{k})^{-1} x_{t}}\label{eqn:HLINUCBCarm}
\end{equation}
For each cluster $i$ and selects the cluster with the highest value of $p^{c}_{t,i}$ (Line 4). This quantity encapsulates the estimated reward from this cluster, along with an uncertainty term. Then, HLINUCBC finds the best arm within this cluster by computing for each arm $k$ in this cluster, the quantity (\ref{eqn:HLINUCBCarm})
and selects the arm with the highest value $p_{t,k}$ (Line 5). 
% As in the description of HUCBC, we use the superscript $c$ to denote per-cluster quantities as opposed to per-arm quantities.

\noindent\textbf{ Historical Information: }
HLINUCBC applies online ridge regression to the historical data at both per-cluster and per-arm levels. The aim is to collect enough information about how the context vectors and rewards relate to each other for each cluster and arm using the historical data, so that it can jump-start the algorithm by achieving a low number of errors at the early stages of the online exploration. At the initialization step, HLINUCBC is seeded with $\mathbf{H}^c_i$ and $\mathbf{H}_k$ which are respectively the history matrices for the clusters and arms. This is in contrast to LINUCB, where the initialization is done using an identity matrix. It is also seeded with the vectors $b^{hc}_i$ and $b^h_k$, which record the weighted sum of historical contexts, with the weight being the rewards. In contrast, in LINUCB, this quantity is initialized to $\mathbf{0}$.
\begin{algorithm}[H]
\caption{HLINUCBC}
\label{alg:HLINUCBC}
\begin{algorithmic}[1]
\REQUIRE{$\alpha \in \mathbb{R}_{>0}$, history matrices $\mathbf{H}^c_i$ for each $i \in \{1, \ldots, C\}$, $\mathbf{H}_k$ for each $k \in \{1, \ldots, K\}$}, weighted sum of contexts $b^{hc}_i$ and $b^h_k$.

 \STATE \textbf{for} arm $k \in \{1, \ldots, K\}$ \textbf{do} $\mathbf{A}_k \gets \mathbf{H}_k$, $b_k \gets b^h_k$
 \STATE \textbf{for} cluster $i \in \{1, \ldots, C\}$ \textbf{do} 
  	$\mathbf{A}^c_i \gets \mathbf{H}^c_i$, $b^c_i \gets b^{hc}_i$  
  \FOR{$t \in \{1, \ldots, T\}$}
    \STATE Choose cluster $i(t) = \argmax_{i} p^{c}_{t,i}$
    \STATE Play arm $k(t) = \argmax_{k \vert c(k) = i(t)} p_{t, k}$
	\STATE $\mathbf{A}_{k(t)} \gets \mathbf{A}_{k(t)} + x_{t} x^{\top}_{t}$,\; $b_{k(t)} \gets b_{k(t)} + r_{k(t)}(t) x_{t}$,\; $\hat{\theta}_{k(t)} \gets \mathbf{A}_{k(t)}^{-1} b_{k(t)}$
	\STATE $\mathbf{A}^c_{i(t)} \gets \mathbf{A}^c_{i(t)} + x_{t} x^{\top}_{t}$,\; $b^c_{i(t)} \gets b^c_{i(t)} + r_{k(t)}(t) x_{t}$,\; $\hat{\theta}^c_{i(t)} \gets \mathbf{A}_{i(t)}^{c-1} b_{i(t)}$
\ENDFOR
\end{algorithmic}
\end{algorithm}

% \noindent\textbf{Regret Analysis: }
% The following theorem establishes the regret of HLINUCB analogous to the known regret bound for LINUCB~\cite{abbasi2011improved}.
% We now upper bound the regret of the HLINUCB, where the HLINUCB is just a version of LINUCB with History. Note that the general setting  \cite{abbasi2011improved} takes one context per arm instead for our setting of the one context share by actions. To upper bound our algorithm in the general setting, we cast our setting as theirs by the following steps.
% We choose a global vector $\hat{\theta}$ as the concatenation of the $K$ vectors, so $\hat{\theta} =[\hat{\theta}_{1},...,\hat{\theta}_{K}]$. We define a context $x_{t,k}$ per action with $x_t$, where $x_{t,k} =[...0,x_{t}^{\top},0,... ]^{\top}$ and $x_t$ being the $k$-th vector within the concatenation. All $\mathbf{A}_t$, $\mathbf{A}^c_t$,$r(t)$, $b_t$, $b^c_t$, $\mathbf{H}$ and $\mathbf{H}^c$ can be similarly defined from $\mathbf{A}_{k(t)}$, $\mathbf{A}^c_{i(t)}$, $r_{k(t)}$, $b_{k(t)}$ and $b^c_{i(t)}$, $\mathbf{H}_k$ and $\mathbf{H}^c_i$ .
\begin{thm} \label{thm:hlinucb}
With probability $1-\delta$, where $0 < \delta < 1$, the upper bound on the R(T) for the HLINUCB in the contextual bandit problem, $K$ arms and $d$ features (context size) is as follows:
\begin{gather*}
R(T)\leq \sigma (\sqrt{d \log (\frac{det(\mathbf{A}_{T})^{1/2} }{\delta\;det(\mathbf{H})^{1/2}} )}+ \frac{||\theta^*||}{\sqrt{\phi}})\sqrt{8\; T \log(\frac{det(\mathbf{A}_{T})}{det(\mathbf{H})})} 
\end{gather*}
with $||x_t||_2 \leq L$  and $ \phi \in R$
\end{thm}
\begin{proof}
We need the following lemma from \cite{abbasi2011improved}, 

\begin{lem} \label{lem:ct} 
Assuming that, the measurement noise $n_t$ is independent of everything and is $\sigma$-sub-Gaussian for some $\sigma >0$, i.e., $E[e^{\phi\, n_t} ] \leq exp(\frac{\phi^2 \sigma^2}{2})$ for all $ \phi \in R$ and $r_t=\langle x_t,\theta^* \rangle+n_t$ .
With probability $1-\delta$, where $0 < \delta < 1$ and  $\theta^*$ lies in the confidence ellipsoid.
\begin{gather*}
C_{t}=\{ \theta: ||\theta-\hat{\theta}_{t}||_{A_{t}} \leq c_{t} \} \nonumber \\
\text{where } c_t \defeq \sigma \sqrt{d \;log (\frac{det(\mathbf{A}_{t})^{1/2} \;det(\mathbf{H})^{-1/2}}{\delta} )}+ \frac{||\theta^*||}{\sqrt{\phi}}\}
\end{gather*}
\end{lem}
 The inner product is denoted by $\langle \cdot , \cdot \rangle$. The lemma is directly adopted from theorem 2 in \cite{abbasi2011improved}, we follow the same step of proof, the main difference is that they have $\mathbf{A}_t=\lambda I+ \sum_{s=1}^t x_sx_s^\top $ and we have,
 \begin{gather*}\mathbf{A}_t=\mathbf{H}+ \sum_{s=1}^t x_sx_s^\top \end{gather*}
Following the same step as the proof of theorem 2 in \cite{abbasi2011improved} We also have the following,
%We consider the high probability event $\theta^* \in C_{t}$ for all $t \geq 0$

%$R(t) = \langle x_t^*,\theta_{t} \rangle - \langle x_t,\theta_{t}^* \rangle$  \;\; $\tilde{\theta}=argmax_{\theta_{t} \in C_t} \langle x, \theta_{t}\rangle$
  
%$ \leq \langle x_t, \tilde{\theta}_{t} \rangle - \langle x_t,\theta_{t}^* \rangle$ \;$\theta_{t}^* \in C_{t}$

%$ = \langle x_t, \tilde{\theta}_{t} -\theta^*_t \rangle$ 

%$ = \langle x_t, \hat{\theta}_{t} -\theta^*_{t} \rangle$ $ \langle x_t, \tilde{\theta}_{t} - \hat{\theta}_{t} \rangle$

%$ \leq ||x_t||_{A_{t}^{-1}} ||\hat{\theta}_{t}-\theta_{t}^*||_{A_{t}}+||x_t||_{A_{t}^{-1}} ||\tilde{\theta}_{t}-\hat{\theta}_{t}||_{A_{t}}$ using Cauchy-Schwarz

$R(t) \leq 2 c_t ||x_t||_{\mathbf{A}_{t}^{-1}}$,
Since $x^{\top}\theta_{t}^* \in [-1,1]$ for all $x \in X_t $ then we have $R(t) \leq 2$. Therefore,
\begin{gather*}
R(t) \leq min\{2c_{t}||x||_{\mathbf{A}^{-1}_{t}},2\} \leq 2c_{t} min\{||x||_{\mathbf{A}^{-1}_{t}},1\}
\end{gather*}
\begin{gather*}
[R(t)]^2 \leq 4 c_{t}^2 min\{||x||^2_{\mathbf{A}^{-1}_{t}},1\}
\end{gather*}
we have, \begin{gather*}R(T) \leq \sqrt{T\sum_{t=1}^{T}[R(t)]^2} \leq \sqrt{  \sum_{t=1}^T 4 c_{t}^2  min\{||x||^{2}_{\mathbf{A}^{-1}_{t}},1\}} \\
R(T)\leq 2 c_{T} \sqrt{ T} \sqrt{ \sum_{t=1}^T   min\{||x||^{2}_{\mathbf{A}^{-1}_{t}},1\}}
\end{gather*} with $c_{t}$ monotonically increasing.
Since $x \leq 2\,log(1+x)$ for $x \in [0,1]$,
\begin{gather*}
\sum_{t=1}^{T} min\{||x_t||^2_{\mathbf{A}_{t}^{-1}}, 1\} \leq 2 \sum_{t=1}^{T} \log(1+||x_t||^2_{\mathbf{A}^{-1}_t})\\\leq 2 (\log(det(\mathbf{A}_{t})-\log(det(\mathbf{H})) 
\end{gather*}
here we also use the fact that have $\mathbf{A}_t=\mathbf{H}+ \sum_{s=1}^t x_sx_s^\top $  to get the last inequality.  
\begin{gather*}
R(T)\leq 2 c_{t} \sqrt{2(\log(det(\mathbf{A}_{t})-\log(det(\mathbf{H}))}
\end{gather*} 
by upper-bounding $c_{T}$ using lemma \ref{lem:ct} we get our result.
\end{proof}
Theorem \ref{thm:hlinucb} shows that the HLINUCB upper bound has $\log(\frac{det(\mathbf{A}_{t})}{det(\mathbf{H})})$ under the square root whereas LINUCB has $\log(det(\mathbf{A}_{t})$. 
Recall from~\cite{abbasi2011improved} that LINUCB has a regret guarantee which is almost the same as the one
in Theorem \ref{thm:hlinucb} except for the $det(\mathbf{H})$ term. We now demonstrate 
that $\frac{det(\mathbf{A}_{t})}{det(\mathbf{H})} \leq det (\mathbf{A}_{t})$ and thereby show that 
HLINUCB has a provably better guarantee than LINUCB.
The matrix $H$ can be written as $I + D_h$ where $I$ is the identity matrix and $D_h$ is the design matrix constructed using the contexts in history. Both $I$ and $D_h$ are real symmetric and hence Hermitian matrices. Further, $D_h$ is positive semi-definite since $D_h = \sum_i x_i x_i^T$, where the $x_i$ are the historical contexts; to see this, note that $\forall y,~y^T D_h y = \sum_i y^T x_i x_i^T y = \sum_i (x_i^T y)^2 \geq 0$. Since all the eigenvalues of $I$ equal $1$ and since all the eigenvalues of $D_h$ are non-negative, by Weyl's inequality in matrix theory for perturbation of Hermitian matrices \cite{thompson1971eigenvalues}, the eigenvalues of $H$ are lower-bounded by $1$. Hence $\mathbf{det(H)}$ which is the product of the eigenvalues of $H$ is lower-bounded by $1$ which proves our claim.
% \begin{lem}~\cite{abbasi2011improved}
% \label{lemma:linucb}
% With probability $1-\delta$, where $0 < \delta < 1$, the upper bound on the R(T) for the LINUCB in the contextual bandit problem, $K$ arms and $d$ features (context size) is given as follows:
% \begin{eqnarray*}
% R(T)\leq \sigma (\sqrt{d \;log (\frac{det(\mathbf{A}_{T})^{1/2} }{\delta\;^{}} )}+ \frac{||\theta^*||}{\sqrt{\phi}})\sqrt{8\; T log(det(\mathbf{A}_{T}))} 
% \end{eqnarray*}
% \end{lem}
\section{META Algorithm}
\label{sec:meta}
Intuitively, HUCBC can be expected to outperform HUCB \textit{when} the quality of clustering is good. However, when the clustering quality is poor, this is far less likely to be the case. 
%This is also supported by the regret analysis of HUCBC. HUCBC has $O(\log T)$ regret over the time horizon $T$ under all clustering conditions; however the constants involved in the regret bound for tight clustering are significantly smaller than those for adversarial clustering. 
In order to choose between these two strategies correctly, we employ the META algorithm. Specifically, for the classical bandit problem, META uses UCB to decide between HUCBC and HUCB \textit{at each iteration}. For the contextual bandit problem, META uses UCB to decide between HLINUCBC and HLINUCB \textit{at each iteration}. The basic idea is that if the clustering quality is poor, the fraction of the time when META plays the strategy with both clustering and history will become vanishingly small, leading to the desired behavior.
\begin{thm}
\label{theo:meta}
The META algorithm is asymptotically optimal for the classical bandit problem under the assumption that the drift conditions of~\cite{kocsis2006bandit} hold for HUCBC.
\end{thm}
 \begin{proof}
 The META algorithm essentially has two arms to play in each iteration, namely HUCB and HUCBC. We know from Fact \ref{fact:hucb} that HUCB has a sublinear regret and is asymptotically optimal. Further, in the proof of Theorem \ref{theo:adversarial}, we also establish that HUCB satisfies the drift conditions of ~\cite{kocsis2006bandit}. Hence, under the assumption that HUCBC also satisfies the drift conditions, \cite{kocsis2006bandit} guarantees that META will also have sub-linear regret and is asymptotically optimal.
 \end{proof}
% Theorem \ref{theo:meta} uses the assumption that HUCBC also satisfies the drift conditions like HUCB. In Theorem \ref{theo:tight}, we show that HUCB satisfies the drift conditions; this provides some evidence to support the assumption that HUCBC also satisfies them.% -- however, we defer the formal proof of this hypothesis to future work.
 
\section{Experimental Evaluation}
\label{sec:experiments}
%In this section, we present the experimental evaluation of our algorithms with synthetic and real world datasets. 
\noindent\textbf{Experiments with Synthetic Classical Bandit:} We compare our proposed strategies HUCBC (Section \ref{sec:HUCBC}) and META (Section \ref{sec:meta}) to the state-of-the-art competitors UCB~\cite{UCB}, HUCB~\cite{shivaswamy2012multi}, and UCBC~\cite{pandey2007multi}.
HUCB is the degenerate version of HUCBC where all arms belong to the same cluster, i.e., no use of clustering information is made only historical data is employed. In contrast, UCBC is the degenerate version which uses the clustering information but no historical data.
\begin{figure*}
    \centering
    \begin{subfigure}[b]{0.23\textwidth}
        \includegraphics[width=\textwidth]{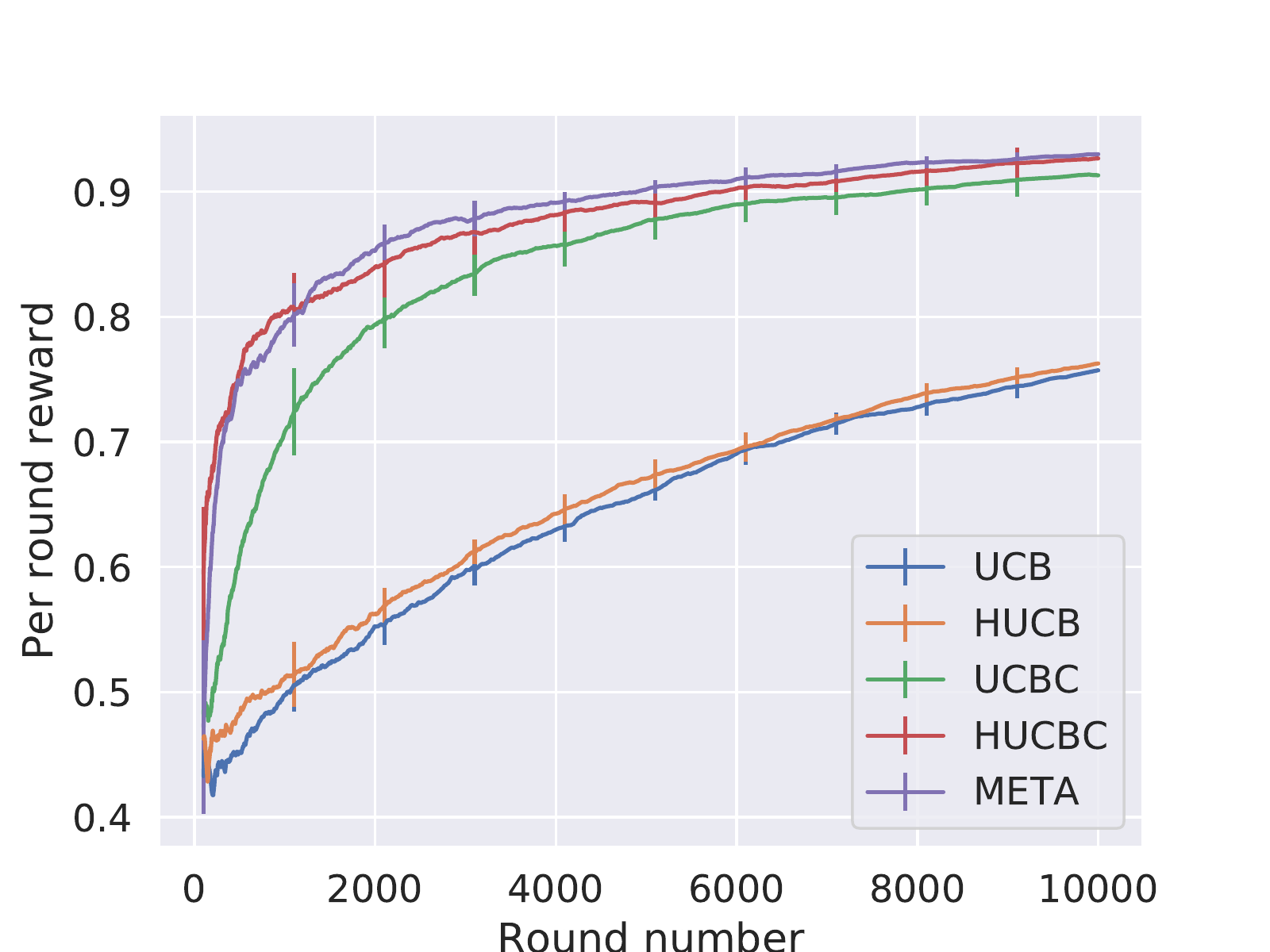}
        \caption{Synthetic classical bandit}
        \label{fig:classicalbandits}
    \end{subfigure}
    ~ %add desired spacing between images, e. g. ~, \quad, \qquad, \hfill etc. 
    \begin{subfigure}[b]{0.23\textwidth}
        \includegraphics[width=\textwidth]{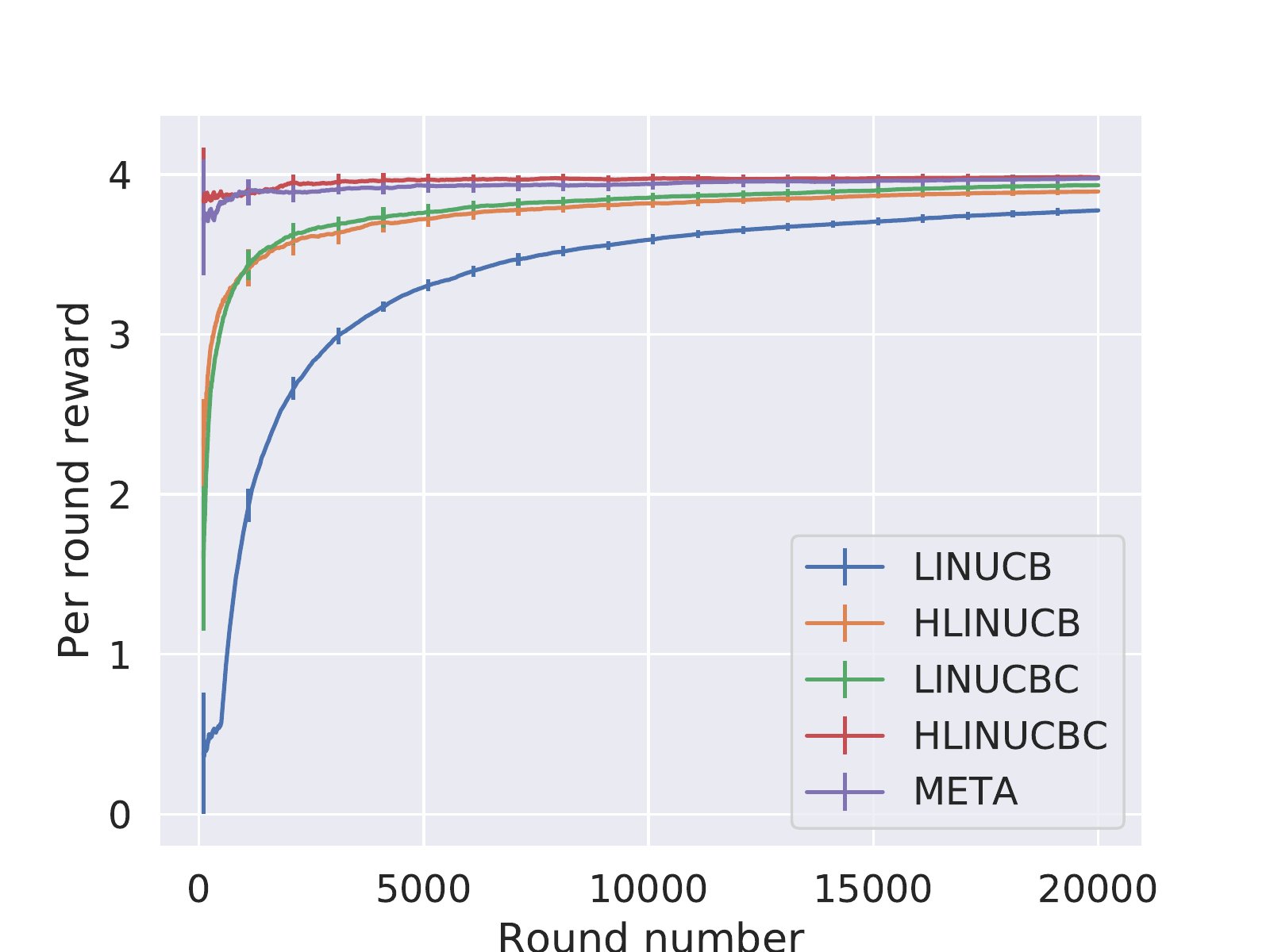}
        \caption{Synthetic contextual bandit}
        \label{fig:contextualbandits}
    \end{subfigure}
    ~ %add desired spacing between images, e. g. ~, \quad, \qquad, \hfill etc. 
    \begin{subfigure}[b]{0.23\textwidth}
        \includegraphics[width=\textwidth]{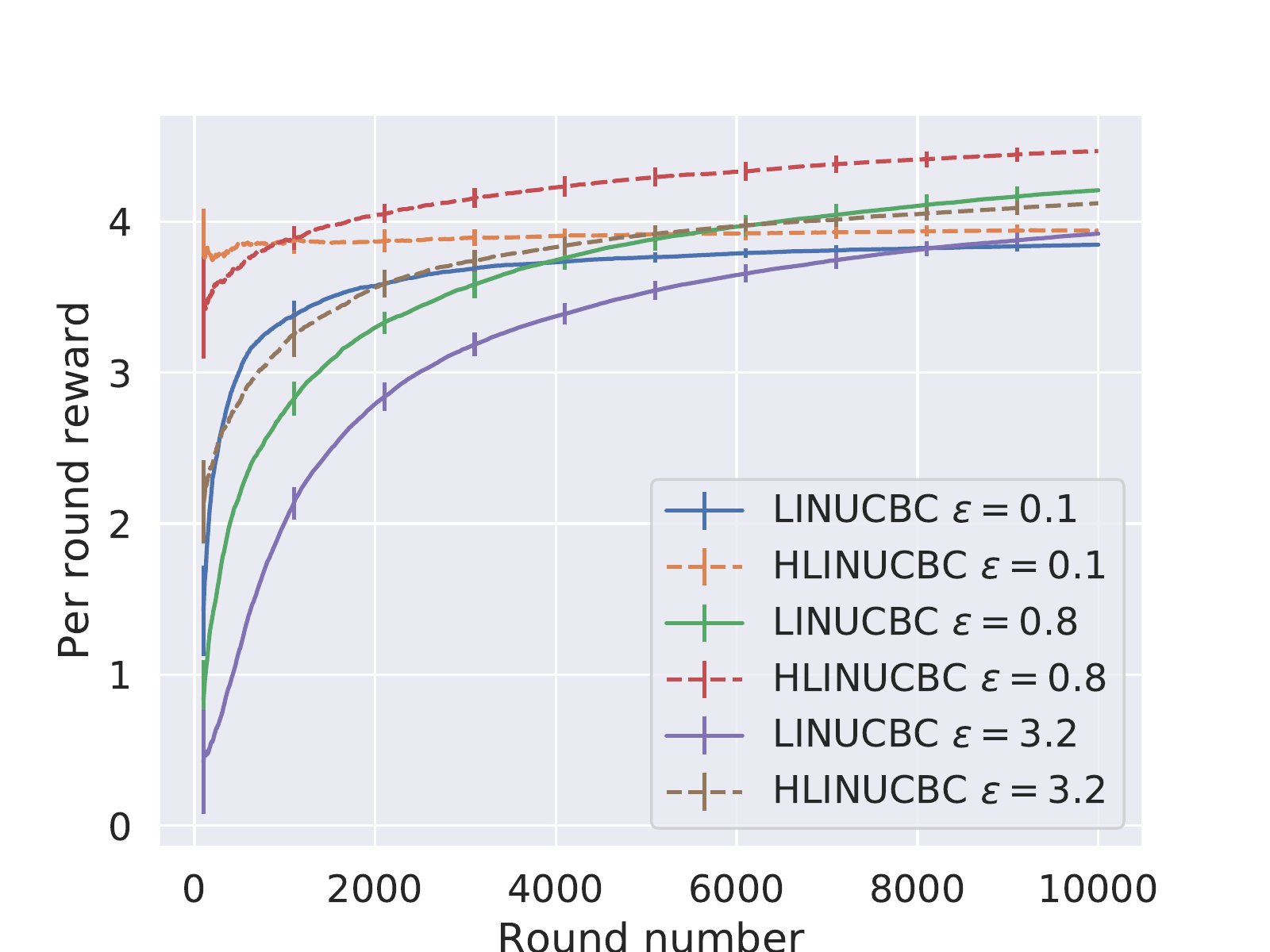}
        \caption{Synthetic CB with radii}
        \label{fig:contextualradii}
    \end{subfigure}
    
    \begin{subfigure}[b]{0.23\textwidth}
        \includegraphics[width=\textwidth]{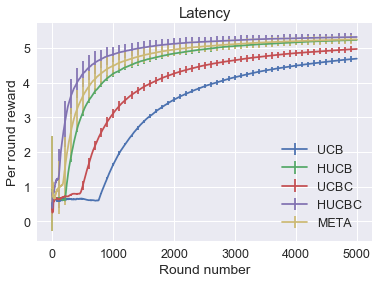}
        \caption{Web server selection with good clustering}
        \label{fig:latency1}
    \end{subfigure}
    ~ %add desired spacing between images, e. g. ~, \quad, \qquad, \hfill etc. 
    \begin{subfigure}[b]{0.23\textwidth}
        \includegraphics[width=\textwidth]{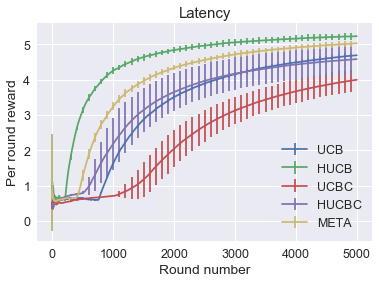}
        \caption{Web server selection with poor clustering}
        \label{fig:latency2}
    \end{subfigure}
    ~ %add desired spacing between images, e. g. ~, \quad, \qquad, \hfill etc. 
    \begin{subfigure}[b]{0.23\textwidth}
        \includegraphics[width=\textwidth, clip]{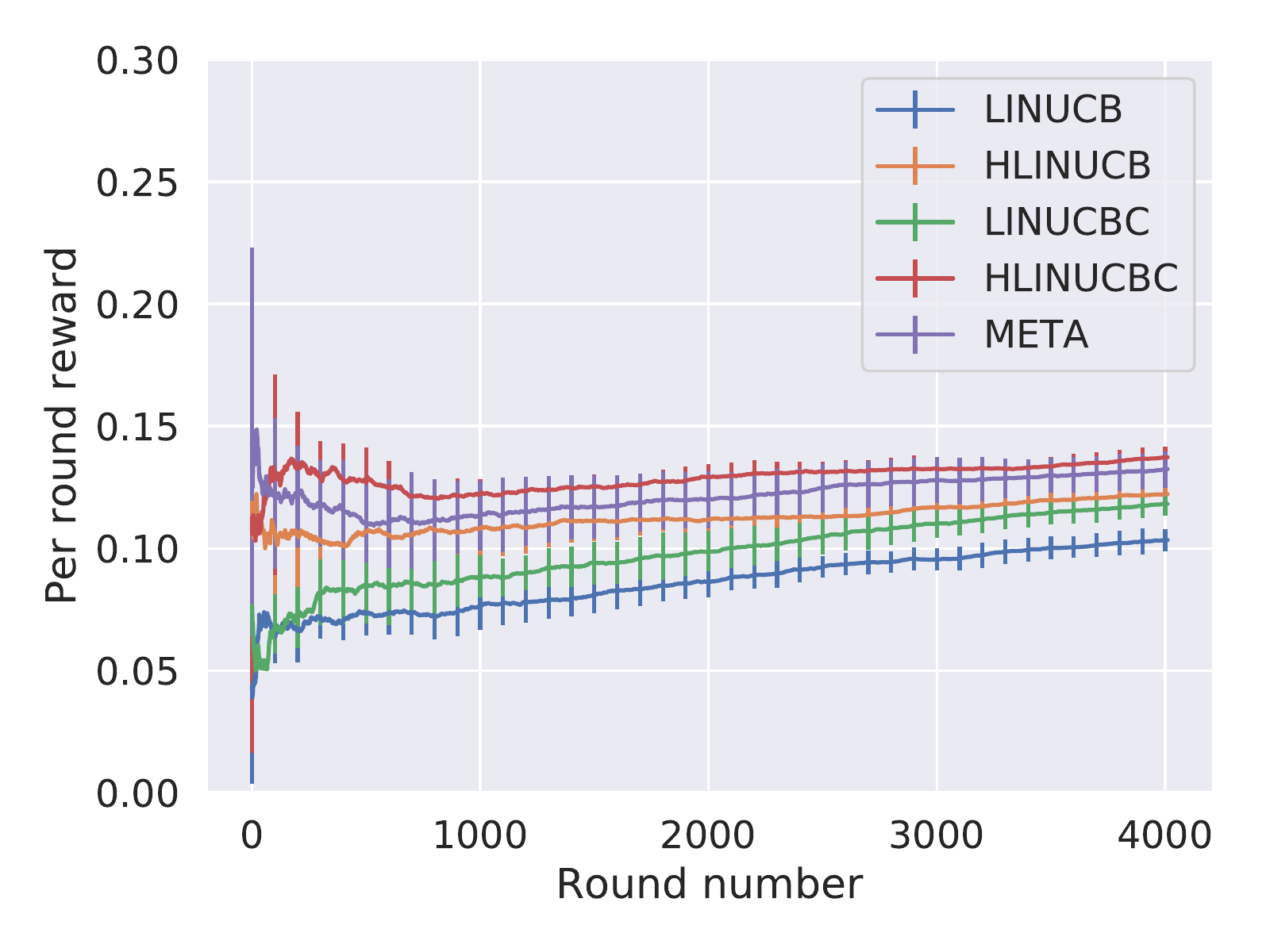}
        \caption{Warfarin dosage}
        \label{fig:classification}
    \end{subfigure}
    \caption{(a) HUCBC outperforms its competitors. META quickly learns to follow HUCBC 
    (b) HLINUCBC outperforms its competitors. META quickly learns to follow HLINUCBC
    (c) Rewards for the different cluster-based methods under different cluster radii $\epsilon$. 
    %The average reward decreases with increasing cluster radii.
    HLINUCBC benefits from historical data in all cases.
    (d) HUCBC provides significantly better rewards than competitors when clustering is good. Meta learns to follow HUCBC
    (e) HUCB outperforms HUCBC when clustering is poor and META learns to follow HUCB
    (f) HLINUCBC outperforms competitors and META learns to follow this}
    \label{fig:allfigures}
\end{figure*}

Our synthetically generated data assumes 10 clusters and 100 arms.
A random permutation assigns each arm $k$ to a cluster $c(k)$, where the centroid of cluster $i$ is defined as $\lambda(i) = \frac{1}{2}\left(u_i + \frac{1}{i}\right)$ and $u_i \sim \mathcal{U}(0, 1)$.
The reward obtained when arm $k$ is played is sampled from $\mathcal{U}(0, 2 \alpha_k \lambda(c(k)))$, where $\alpha_k$ is a arm dependent constant. Thus, the expected reward for cluster $i$ is $\lambda(i)$. The historical data is generated as follows: for 25\% we generate data playing them $X$ times where $X$ is sampled from a Poisson distribution with parameter 10.
We report the results of 20 trials in in Figure \ref{fig:classicalbandits}.
Each trial consisted of $10^4$ rounds and we plot the the number of rounds per-round-reward (cumulative reward until that round / number of rounds).
The error bars correspond to the interval of $\pm 1$ standard deviation around the mean. 
Clearly, the per-round-reward of HUCBC is the fastest to converge to that of the best arm. Interestingly, while historical 
data improves the performance of HUCB over UCB only to
a small degree (possibly due to asymmetry in history), combining this information with the clustering information improves
the performance of HUCBC over UCBC to a significantly larger degree. This is an effect of the fact that even though
historical data is sparse and asymmetric at the arm level, combining this information at the cluster level is more impactful. META converges rather quickly to the correct choice (HUCBC) in this setting.

\noindent\textbf{Experiments with Synthetic Contextual Bandit:} We will now compare our proposed contextual bandits HLINUCBC (Section \ref{sec:HLINUCBC}) and the META algorithm selection to LINUCB~\cite{13}, HLINUCB and LINUCBC.
Similarly to HUCB and UCBC, HLINUCB and LINUCBC are degenerated versions of HLINUCBC using either only the historical data or the clustering information.

The synthetic data was created by assuming 10 clusters and 100 arms.
Again, a random permutation assigns each arm $k$ to a cluster $c(k)$.
The centroid $\theta^c_i$ of cluster $i$ is sampled from $\mathcal{N}(\mathbf{0}, I_5)$.
The coefficient $\theta_k$ of arm $k$ is defined as $\theta_k = \theta^c_i + \epsilon \nu_k$, where $\nu_k$ is sampled independently at random from $\mathcal{N}(\mathbf{0}, I_5)$. The reward for playing arm $k$ in context $x$ is sampled from $\mathcal{U}(0, 2\theta_k^{\top} x)$.
Thus, our synthetically generated problem is linear in the input and the expected distance between clusters is $\sqrt{5}\epsilon$.
By varying $\epsilon$, we control the tightness of the clusters and can observe the impact of clustering information with varying quality of clusters.
Contexts in each round of a trial are drawn i.i.d from the multivariate normal distribution $\mathcal{N}(\mathbf{0}, I_5)$.
We created asymmetric historical data for arms by playing an arm $X$ times with $X$ distinct contexts,
where $X \sim Poisson(10)$.

We report the results for $\epsilon = 0.1$ in Figure \ref{fig:contextualbandits}. 
As in the case of classical bandits, we repeat the experiment 20 times, where each trial consists of 10,000 rounds.
The error bars correspond to the interval of $\pm 1$ standard deviation around the mean. 
HLINUCBC is clearly outperforming its competitors and META is able to identify HLINUCBC as the stronger method over HLINUCB.
In this set up, historical data and clustering are equally important such that HLINUCB and LINUCBC provide similar results.
LINUCB clearly provides the worst results.
We also compare the performance of HLINUCBC vs. LINUCBC under different values of $\epsilon \in \{0.1, 0.8, 3.2\}$.
We normalize the rewards in these three distinct settings to enable a meaningful comparison.
We present the results in Figure \ref{fig:contextualradii}.
Since HLINUCBC utilizes historical data in addition to the
clustering information, we can see that it improves upon the performance of LINUCBC for every setting of $\epsilon$.

%\noindent\textbf{Experiments with Real-World Data: }
%We compare our bandit algorithms on two different real-world problems.
%The classical bandits are compared on the task of web server selection while the contextual bandits are analyzed for the Warfarin drug dosage problem. 

\noindent\textbf{Web Server Selection:~\cite{CDN} }
%We evaluate the classical bandits on the problem which is commonly known as the Content Distribution Network.
Essentially one needs to decide from what source one should pull content available at multiple sources while minimizing the cumulative latency for successive retrievals.
Similar to \cite{Vermorel2005}, we assume that only a single resource can be picked at a time.
The university web page data set\footnote{https://sourceforge.net/projects/bandit/} features more than $700$ sources with about $1300$ response times each. 
In order to facilitate our clustering approach in this setting, we perform two types of clustering: 1) split the resources based on average latencies into five categories ranging from `fast' to `slow', 2) based on domain names into 17 clusters.
Historical data was generated by choosing 200 observations in total at random.
In Figures~\ref{fig:latency1} and \ref{fig:latency2} we show the mean cumulative reward for the UCB, UCB with history (HUCB), UCB with clustering (UCBC), UCB with both (HUCBC), and META over 10 repetitions.
Before each experiment, the data has been shuffled. We can observe a clear impact of employing clustering and history and the combination of both, and the type of clustering. The first type of clustering achieves an effective grouping of arms while the second one does not. Confirming our observations on the synthetic data, only clustering information provides better results than only historic data when clustering is of good quality. The two kinds of clustering also highlight the usefulness of our META approach that in both cases converges to the correct approach.

\noindent\textbf{Warfarin Drug Dosage: } 
The Warfarin problem data set is concerned with determining the correct initial dosage of the drug Warfarin for a given patient.
Correct dosage is variable due to patient’s clinical, demographic and genetic context. 
We select this benchmark because it enables us to compare the different algorithms since this problem allows us to create a hierarchical classification data set. Originally, the data set was converted to a classification data set with three classes using the nominal feedback of the patient.
We create a hierarchy as follows.
We divide each of the original three classes into five more granular classes.
Bandit methods which do not use hierarchy will simply tackle this problem as a classification task with 15 classes.
Others will first assign instances to one of the three classes and finally decide to choose one of the final five options.
The order of the patients is shuffled, 1500 of them are chosen as historical data.
We report the mean and standard deviation of 10 runs in Figure \ref{fig:classification}.
HLINUCBC is outperforming all competitor methods and also META is able to successfully detect that HLINUCBC is the dominating method.
For this problem HLINUCB provides better results than LINUCBC, indicating that historic data is worth more than clustering information.
%All methods clearly outperform LINUCB.

\section{Conclusions}
We introduced a variety of algorithms for classical and contextual bandits which incorporate historical observations and pre-clustering information between arms in a principled manner.
We demonstrated their effectiveness and robustness both through rigorous regret analysis as well as extensive experiments on synthetic and real world datasets. Two interesting open problems emerge from this work: 1) Are there instance independent regret bounds for HUCBC which depend only on the total number of clusters and the number of arms in the optimal cluster? 2) What are the upper and lower bounds on the regret of HLINUCBC?

%\clearpage
\bibliographystyle{aaai}
\bibliography{ijcai19}

\begin{thebibliography}{}

\bibitem[\protect\citeauthoryear{Abbasi-Yadkori, P{\'a}l, and
  Szepesv{\'a}ri}{2011}]{abbasi2011improved}
Abbasi-Yadkori, Y.; P{\'a}l, D.; and Szepesv{\'a}ri, C.
\newblock 2011.
\newblock Improved algorithms for linear stochastic bandits.
\newblock In {\em Advances in Neural Information Processing Systems},
  2312--2320.

\bibitem[\protect\citeauthoryear{Allesiardo, F{\'{e}}raud, and
  Bouneffouf}{2014}]{AllesiardoFB14}
Allesiardo, R.; F{\'{e}}raud, R.; and Bouneffouf, D.
\newblock 2014.
\newblock A neural networks committee for the contextual bandit problem.
\newblock In {\em Neural Information Processing - 21st International
  Conference, {ICONIP} 2014, Kuching, Malaysia, November 3-6, 2014.
  Proceedings, Part {I}},  374--381.

\bibitem[\protect\citeauthoryear{Alon and Spencer}{2016}]{alonspencer:prob}
Alon, N., and Spencer, J.~H.
\newblock 2016.
\newblock {\em The Probabilistic Method}.
\newblock Wiley Publishing, 4th edition.

\bibitem[\protect\citeauthoryear{Auer and Cesa-Bianchi}{1998}]{AuerC98}
Auer, P., and Cesa-Bianchi, N.
\newblock 1998.
\newblock On-line learning with malicious noise and the closure algorithm.
\newblock {\em Ann. Math. Artif. Intell.} 23(1-2):83--99.

\bibitem[\protect\citeauthoryear{Auer, Cesa-Bianchi, and Fischer}{2002}]{UCB}
Auer, P.; Cesa-Bianchi, N.; and Fischer, P.
\newblock 2002.
\newblock Finite-time analysis of the multiarmed bandit problem.
\newblock {\em Machine Learning} 47(2-3):235--256.

\bibitem[\protect\citeauthoryear{Balakrishnan \bgroup et al\mbox.\egroup
  }{2019}]{DB2019}
Balakrishnan, A.; Bouneffouf, D.; Mattei, N.; and Rossi, F.
\newblock 2019.
\newblock Incorporating behavioral constraints in online {AI} systems.
\newblock {\em AAAI 2019}.

\bibitem[\protect\citeauthoryear{Bouneffouf and Rish}{2019}]{surveyDB}
Bouneffouf, D., and Rish, I.
\newblock 2019.
\newblock A survey on practical applications of multi-armed and contextual
  bandits.
\newblock {\em CoRR} abs/1904.10040.

\bibitem[\protect\citeauthoryear{Bouneffouf \bgroup et al\mbox.\egroup
  }{2017}]{BRCF17}
Bouneffouf, D.; Rish, I.; Cecchi, G.~A.; and F{\'{e}}raud, R.
\newblock 2017.
\newblock Context attentive bandits: Contextual bandit with restricted context.
\newblock In {\em IJCAI 2017, Melbourne, Australia, August 19-25, 2017},
  1468--1475.

\bibitem[\protect\citeauthoryear{Chu \bgroup et al\mbox.\egroup
  }{2011}]{ChuLRS11}
Chu, W.; Li, L.; Reyzin, L.; and Schapire, R.~E.
\newblock 2011.
\newblock Contextual bandits with linear payoff functions.
\newblock In Gordon, G.~J.; Dunson, D.~B.; and Dudik, M., eds., {\em AISTATS},
  volume~15 of {\em JMLR Proceedings},  208--214.
\newblock JMLR.org.

\bibitem[\protect\citeauthoryear{Gentile \bgroup et al\mbox.\egroup
  }{2017}]{Gentile2017}
Gentile, C.; Li, S.; Kar, P.; Karatzoglou, A.; Zappella, G.; and Etrue, E.
\newblock 2017.
\newblock On context-dependent clustering of bandits.
\newblock In {\em Proceedings of the 34th International Conference on Machine
  Learning, {ICML} 2017, Sydney, NSW, Australia, 6-11 August 2017},
  1253--1262.

\bibitem[\protect\citeauthoryear{Kocsis and
  Szepesv{\'a}ri}{2006}]{kocsis2006bandit}
Kocsis, L., and Szepesv{\'a}ri, C.
\newblock 2006.
\newblock Bandit based monte-carlo planning.
\newblock In {\em ECML}, volume~6,  282--293.
\newblock Springer.

\bibitem[\protect\citeauthoryear{Korda, Sz{\"{o}}r{\'{e}}nyi, and
  Li}{2016}]{Korda2016}
Korda, N.; Sz{\"{o}}r{\'{e}}nyi, B.; and Li, S.
\newblock 2016.
\newblock Distributed clustering of linear bandits in peer to peer networks.
\newblock In {\em Proceedings of the 33nd International Conference on Machine
  Learning, {ICML} 2016, New York City, NY, USA, June 19-24, 2016},
  1301--1309.

\bibitem[\protect\citeauthoryear{Krishnamurthy, Wills, and Zhang}{2001}]{CDN}
Krishnamurthy, B.; Wills, C.; and Zhang, Y.
\newblock 2001.
\newblock On the use and performance of content distribution networks.
\newblock In {\em Proceedings of the 1st ACM SIGCOMM Workshop on Internet
  Measurement}, IMW '01,  169--182.
\newblock New York, NY, USA: ACM.

\bibitem[\protect\citeauthoryear{Lai and Robbins}{1985}]{LR85}
Lai, T.~L., and Robbins, H.
\newblock 1985.
\newblock Asymptotically efficient adaptive allocation rules.
\newblock {\em Advances in Applied Mathematics} 6(1):4--22.

\bibitem[\protect\citeauthoryear{Langford and Zhang}{2007}]{LangfordZ07}
Langford, J., and Zhang, T.
\newblock 2007.
\newblock The epoch-greedy algorithm for multi-armed bandits with side
  information.
\newblock In Platt, J.~C.; Koller, D.; Singer, Y.; and Roweis, S.~T., eds.,
  {\em NIPS}.
\newblock Curran Associates, Inc.

\bibitem[\protect\citeauthoryear{Li \bgroup et al\mbox.\egroup
  }{2010a}]{newsrecommendation}
Li, L.; Chu, W.; Langford, J.; and Schapire, R.~E.
\newblock 2010a.
\newblock A contextual-bandit approach to personalized news article
  recommendation.
\newblock In {\em Proceedings of the 19th International Conference on World
  Wide Web}, WWW '10,  661--670.
\newblock New York, NY, USA: ACM.

\bibitem[\protect\citeauthoryear{Li \bgroup et al\mbox.\egroup }{2010b}]{13}
Li, L.; Chu, W.; Langford, J.; and Schapire, R.~E.
\newblock 2010b.
\newblock A contextual-bandit approach to personalized news article
  recommendation.
\newblock In {\em Proceedings of the 19th international conference on World
  wide web}, WWW '10,  661--670.
\newblock USA: ACM.

\bibitem[\protect\citeauthoryear{Lin \bgroup et al\mbox.\egroup }{2018}]{dj10}
Lin, B.; Bouneffouf, D.; Cecchi, G.~A.; and Rish, I.
\newblock 2018.
\newblock Contextual bandit with adaptive feature extraction.
\newblock In {\em 2018 {IEEE} International Conference on Data Mining
  Workshops, {ICDM} Workshops, Singapore, Singapore, November 17-20, 2018},
  937--944.

\bibitem[\protect\citeauthoryear{Liu \bgroup et al\mbox.\egroup
  }{2010}]{Liu:2010:UIC:1933307.1934597}
Liu, Y.; Li, Z.; Xiong, H.; Gao, X.; and Wu, J.
\newblock 2010.
\newblock Understanding of internal clustering validation measures.
\newblock In {\em Proceedings of the 2010 IEEE International Conference on Data
  Mining}, ICDM '10,  911--916.
\newblock Washington, DC, USA: IEEE Computer Society.

\bibitem[\protect\citeauthoryear{Maillard and Mannor}{2014}]{Maillard2014}
Maillard, O., and Mannor, S.
\newblock 2014.
\newblock Latent bandits.
\newblock In {\em Proceedings of the 31th International Conference on Machine
  Learning, {ICML} 2014, Beijing, China, 21-26 June 2014},  136--144.

\bibitem[\protect\citeauthoryear{Nguyen and Lauw}{2014}]{Nguyen2014}
Nguyen, T.~T., and Lauw, H.~W.
\newblock 2014.
\newblock Dynamic clustering of contextual multi-armed bandits.
\newblock In {\em Proceedings of the 23rd {ACM} International Conference on
  Conference on Information and Knowledge Management, {CIKM} 2014, Shanghai,
  China, November 3-7, 2014},  1959--1962.

\bibitem[\protect\citeauthoryear{Noothigattu \bgroup et al\mbox.\egroup
  }{2018}]{RLbd2018}
Noothigattu, R.; Bouneffouf, D.; Mattei, N.; Chandra, R.; Madan, P.; Varshney,
  K.~R.; Campbell, M.; Singh, M.; and Rossi, F.
\newblock 2018.
\newblock Interpretable multi-objective reinforcement learning through policy
  orchestration.
\newblock {\em CoRR} abs/1809.08343.

\bibitem[\protect\citeauthoryear{Pandey, Chakrabarti, and
  Agarwal}{2007}]{pandey2007multi}
Pandey, S.; Chakrabarti, D.; and Agarwal, D.
\newblock 2007.
\newblock Multi-armed bandit problems with dependent arms.
\newblock In {\em Proceedings of the 24th international conference on Machine
  learning},  721--728.
\newblock ACM.

\bibitem[\protect\citeauthoryear{Sharabiani \bgroup et al\mbox.\egroup
  }{2015}]{sharabiani2015revisiting}
Sharabiani, A.; Bress, A.; Douzali, E.; and Darabi, H.
\newblock 2015.
\newblock Revisiting warfarin dosing using machine learning techniques.
\newblock {\em Computational and mathematical methods in medicine} 2015.

\bibitem[\protect\citeauthoryear{Shivaswamy and
  Joachims}{2012}]{shivaswamy2012multi}
Shivaswamy, P., and Joachims, T.
\newblock 2012.
\newblock Multi-armed bandit problems with history.
\newblock In {\em Artificial Intelligence and Statistics},  1046--1054.

\bibitem[\protect\citeauthoryear{Tang \bgroup et al\mbox.\egroup
  }{2013}]{Tang2013AutomaticAF}
Tang, L.; Rosales, R.; Singh, A.; and Agarwal, D.
\newblock 2013.
\newblock Automatic ad format selection via contextual bandits.
\newblock In {\em CIKM}.

\bibitem[\protect\citeauthoryear{Thompson and
  Freede}{1971}]{thompson1971eigenvalues}
Thompson, R.~C., and Freede, L.~J.
\newblock 1971.
\newblock On the eigenvalues of sums of hermitian matrices.
\newblock {\em Linear Algebra and Its Applications} 4(4):369--376.

\bibitem[\protect\citeauthoryear{Vermorel and Mohri}{2005}]{Vermorel2005}
Vermorel, J., and Mohri, M.
\newblock 2005.
\newblock Multi-armed bandit algorithms and empirical evaluation.
\newblock In {\em Machine Learning},  437--448.

\bibitem[\protect\citeauthoryear{Wang, Zhou, and Shen}{2018}]{wang2018regional}
Wang, Z.; Zhou, R.; and Shen, C.
\newblock 2018.
\newblock Regional multi-armed bandits.
\newblock {\em arXiv preprint arXiv:1802.07917}.

\end{thebibliography}

\end{document}